\documentclass[preprint,12pt,authoryear]{elsarticle}

\usepackage{amssymb}
\usepackage{amsmath}

\usepackage{amsthm}
\newtheorem{theorem}{Theorem}

\usepackage{booktabs} 
\usepackage{algorithm}
\usepackage{multirow}
\usepackage[normalem]{ulem} 
\usepackage{xcolor}         

\usepackage[utf8]{inputenc}
\usepackage{newunicodechar}
\newunicodechar{−}{\textminus}

\usepackage{hyperref}
\usepackage{textcomp}
\usepackage{stfloats}
\usepackage{url}
\usepackage{verbatim}
\usepackage{graphicx}
\usepackage{amsmath,amsfonts}
\usepackage{algorithmic}
\usepackage{array}
\usepackage{makecell}
\usepackage{rotating}
\usepackage{pdflscape}

\journal{}

\begin{document}

\begin{frontmatter}



\title{Rethinking Adam for Time Series Forecasting: A Simple Heuristic to Improve Optimization under Distribution Shifts} 


\author{Yuze Dong$^{a}$,
        Jinsong Wu$^{a,b,*}$}

\affiliation{$^{a}$organization={School of Artificial Intelligence, Guilin University of Electronic Technology},
            addressline={1 Jinji Road}, 
            city={Guilin},
            postcode={541004},
            state={Guangxi},
            country={China}}

\affiliation{$^{b}$organization={Department of Electrical Engineering, University of Chile},
            addressline={Av. Tupper 2007}, 
            city={Santiago},
            postcode={8370451},
            country={Chile}}

\cortext[*]{Corresponding author:\\
\textit{Email address}: wujs@ieee.org (Jinsong Wu)}

\begin{abstract}

Time-series forecasting often faces challenges from non-stationarity, particularly distributional drift, where the data distribution evolves over time. This dynamic behavior can undermine the effectiveness of adaptive optimizers, such as Adam, which are typically designed for stationary objectives. In this paper, we revisit Adam in the context of non-stationary forecasting and identify that its second-order bias correction limits responsiveness to shifting loss landscapes. To address this, we propose TS\_Adam, a lightweight variant that removes the second-order correction from the learning rate computation. This simple modification improves adaptability to distributional drift while preserving the optimizer core structure and requiring no additional hyperparameters. TS\_Adam integrates easily into existing models and consistently improves performance across long- and short-term forecasting tasks. On the ETT datasets with the MICN model, it achieves an average reduction of 12.8\% in MSE and 5.7\% in MAE compared to Adam. These results underscore the practicality and versatility of TS\_Adam as an effective optimization strategy for real-world forecasting scenarios involving non-stationary data. Code is available at: \url{https://github.com/DD-459-1/TS_Adam}.

\end{abstract}


\begin{highlights}
\item Identifies a key limitation of the Adam optimizer in training neural networks for non-stationary time-series forecasting.

\item Proposes TS\_Adam, a lightweight and efficient optimizer that enhances adaptability to distributional drift by a simple yet effective modification.

\item The method requires no extra hyperparameters and can be easily integrated into existing forecasting models as a drop-in replacement for Adam.

\item Extensive experiments on benchmarks demonstrate consistent improvements across various neural network architectures (e.g., MICN).
\end{highlights}


\end{frontmatter}




\section{Introduction}
\label{sec:Introduction}

Time series forecasting plays a critical role in various data-driven applications, including prediction of energy consumption, financial market analysis, and climate modeling. Traditional statistical methods, such as ARIMA and exponential smoothing, rely on parametric assumptions and linear models, which often struggle to capture the nonlinear dynamics and evolving temporal patterns of real-world data. In contrast, deep learning approaches have achieved significant improvements in forecast accuracy by learning hierarchical and non-linear representations. Recent studies indicate that such models can reduce prediction errors by 15–20\% compared to classical statistical baselines~\cite{benidis2022deep}\cite{lim2021time}\cite{miller2024survey}\cite{kong2025deep}.

Despite these advancements, time series forecasting remains challenging due to non-stationarity, particularly distributional drift, where the underlying data distribution evolves over time. This dynamic behavior often leads to degradation in model performance, especially over long horizons or under abrupt environmental changes.

Although most recent efforts have focused on addressing non-stationarity through architectural innovations, loss function design, or signal transformations, relatively little attention has been given to the role of the optimizer in this context. Adaptive optimizers such as Adam are widely used in deep learning, but their behavior under nonstationary conditions has not been thoroughly examined. Empirical evidence suggests that Adam’s second-order bias correction, while beneficial in early-stage convergence, can hinder optimizer responsiveness to evolving objectives in non-stationary environments.

In this work, we address the limitations of adaptive optimizers under distributional drift by introducing TS\_Adam, a simple variant of Adam. This modification improves the adaptability of the optimizer to non-stationary objectives while preserving its core structure and requiring no additional hyperparameters. Extensive experiments on long- and short-term forecasting tasks, including the ETT and M4 datasets, demonstrate that TS\_Adam consistently outperforms standard optimizers. For example, when applied to the MICN model on the ETT datasets, TS\_Adam achieves an average reduction of 12.8\% in MSE and 5.7\% in MAE compared to Adam.

In summary, this paper makes the following key contributions:

\begin{itemize}
    \item We identify a gap in addressing non-stationarity from the optimization perspective and analyze the limitations of Adam under distributional drift.
    \item We propose TS\_Adam, a lightweight variant that improves adaptability in time series forecasting with minimal implementation cost.
    \item We validate the effectiveness of TS\_Adam through extensive experiments in multiple datasets and model architectures.
\end{itemize}

\section{Related works}
\label{sec:Related work}
Non-stationarity in time series, characterized by changing statistical properties such as mean, variance, and autocorrelation, poses significant challenges to model generalization. To mitigate these effects, recent deep learning research has primarily advanced in three directions: architectural design, loss function modification, and signal transformation.

One line of research centers on architectural innovations that aim to improve temporal adaptability. For example, PatchTST~\cite{nie2022time} introduces a fixed-length patching mechanism to emphasize locally stationary segments. MICN~\cite{wang2023micn} utilizes multibranch convolutional layers to separately model short-term and long-term temporal components. SegRNN~\cite{lin2023segrnn} applies segment-wise recurrence to reduce cumulative error in long-sequence modeling. In addition, models such as iTransformer~\cite{liu2023itransformer}, TimeXer~\cite{wang2024timexer}, TimesNet~\cite{wu2022timesnet} incorporate various structural mechanisms to better accommodate distributional shifts, GinAR+~\cite{yu2025ginar+} introduces an end-to-end framework that jointly addresses missing value imputation and spatial correlation reconstruction in multivariate time series via interpolation attention and adaptive graph convolution, MGSFformer~\cite{yu2025mgsfformer} leverages a multi-granularity fusion architecture to effectively integrate heterogeneous data from different sampling intervals for air quality prediction, and BasicTS+~\cite{shao2024exploring} establishes a unified benchmarking framework and, through a novel analysis of dataset heterogeneity, provides a crucial explanation for the seemingly contradictory findings in multivariate time series forecasting (MTSF) research.

Another stream of work focuses on enhancing robustness by refining the loss function. RobustTSF~\cite{cheng2024robusttsf} introduces trend-aware anomaly scoring to selectively filter training samples, while Hounie et al.~\cite{hounie2024loss} propose loss-shaping constraints that mitigate error propagation over extended forecast horizons.

A third approach involves transforming the input signal to reduce non-stationarity prior to modeling. For instance, DERITS~\cite{fan2024deep} applies a frequency derivative transformation to suppress low-frequency trend components, and FAN~\cite{ye2024frequency} implements adaptive normalization in the frequency domain to stabilize inputs across diverse model architectures.

Although these approaches have advanced forecasting performance, they focus primarily on model architecture or input representation. In contrast, the optimization process, particularly the behavior of adaptive optimizers under non-stationary conditions, has received comparatively limited attention. Although a few studies in other domains, such as reinforcement learning, have explored modifications to the Adam algorithm to improve responsiveness to dynamic objectives~\cite{asadi2023resetting, ellis2024adam}, such adaptations remain rare in time series forecasting. Commonly used optimizers such as Adam and AdamW are still applied without modification, despite the inherent non-stationarity of the task.

Despite increasing interest in handling non-stationarity, relatively little attention has been devoted to understanding the role of optimizers in this context. This study aims to bridge this gap by exploring the interaction between optimization dynamics and distributional drift in forecasting.
\section{Impact of Temporal Non-Stationarity on Optimization}
\label{sec:Impact of Temporal NonStationarity on Optimization}

Building on the empirical and architectural challenges reviewed above, we now turn to the theoretical underpinnings of optimization under non-stationary time series settings.

\subsection{Time-Dependence as a Key Driver of Temporal Non-Stationarity}

To simplify the analysis, we focus on the univariate case, which more clearly illustrates the core structural dynamics. We adopt the classical seasonal trend decomposition using Loess (STL) and model the time series $\{Y_t\}_{t=1}^{T}$ as follows. The subsequent analysis applies similarly to the multivariate setting.
\begin{equation}
    Y_t \;=\; T_t \;+\; S_t \;+\; R_t, \qquad t=1,\dots,T,
\end{equation}
where $T_t$ denotes the trend component and satisfies $T_t = T_{t-1} + \varepsilon_t$ for $t = 2,\dots,T$, with $T_1 = c \in \mathbb{R}$ and $\varepsilon_t \overset{\text{i.i.d.}}{\sim} \mathcal{N}(0, \sigma_\varepsilon^2)$; we assume that $\sigma_\varepsilon^2$ is bounded to reflect the phenomenon that real-world trends typically evolve smoothly over time. $S_t$ is the seasonal component, satisfying $S_{t+P} = S_t$ (where $P \in \mathbb{N}^+$ is the seasonal period) and $\sum_{i=1}^P S_i = 0$, where the zero-sum constraint ensures identifiability and separation from the trend. We do not impose continuity or smoothness constraints on $S_t$, as seasonal patterns in the real world often exhibit locally fast variations and abrupt changes. $R_t$ is the remaining component with $R_t \overset{\text{i.i.d.}}{\sim} \mathcal{N}(0, \sigma_R^2)$. The sequences ${\varepsilon_t}$ and ${R_t}$ are mutually independent and both independent of ${S_t}$. This decomposition and its assumptions follow standard practice and align well with the empirical properties of seasonal time series.

\begin{theorem}\label{thm:timevarying}
Under the above decomposition, each observation $Y_t$ follows
a univariate Gaussian distribution whose mean and variance are explicit
functions of time~$t$:
\[
    Y_t \sim \mathcal{N}\bigl(\mu(t),\, \sigma^{2}(t)\bigr),
\]
\[
    \text{where} \quad
    \mu(t) = c + S_t,\quad
    \sigma^2(t) = (t-1)\sigma_\varepsilon^2 + \sigma_R^2.
\]

Consequently, both the first and second moments vary with $t$, rendering the process intrinsically \emph{nonstationary} with an explicitly time-dependent structure.

\end{theorem}

\begin{proof}
See Appendix~A.
\end{proof}

Based on the above analysis, the time-varying nature of the series emerges as a key factor contributing to its non-stationarity. As a result, any forecasting model must approximate a high-dimensional function that evolves over time to capture the continuously changing statistical characteristics of the data. The forecasting task is therefore not about fitting a fixed input-output relation but about modeling a complex time-varying mapping.

\subsection{Limitations of Adam in Learning Time-Varying Functions}

We next analyze the behavior of Adam when learning time-varying objective functions. To this end, we adopt dynamic regret bounds, a standard tool widely used to evaluate optimization performance \cite{zhang2018adaptive}.

The time-varying nature established in Theorem~\ref{thm:timevarying} implies that optimal forecasting parameters must also evolve to track the changing data distribution. To analyze this formally, we adopt the \emph{dynamic regret} framework, which measures the cumulative loss relative to a sequence of time-varying comparators. We consider the parameters of the model $\theta \in \mathbb{R}^d$ and let $f_t(\theta)$ denote the forecast loss at time $t$. The dynamic regret is $R(T) = \sum_{t=1}^T [f_t(\theta_t) - f_t(\theta_t^*)]$, where $\theta_t$ are the parameters of the algorithm, and $\theta_t^* \in \arg\min_{\theta} f_t(\theta)$ are the per-round minimizers.

\begin{theorem}[Dynamic Regret Bound]
Assume that each objective function $f_t(\theta)$ is convex, the gradients are uniformly bounded as $\|\nabla f_t(\theta)\| \leq G$ for all $\theta$ and $t$, the per-round minimizers satisfy $\|\theta_{t+1}^* - \theta_t^*\| \leq \delta$ for all $t$. Let $\{\eta_t\}_{t=1}^T$ denote the sequence of effective learning rates, which is determined by a base learning rate $\alpha$ and the adaptive scaling mechanisms of the optimizer (such as first- or second-order moment estimates and bias corrections). Then, the dynamic regret satisfies:

\begin{equation}
\label{eq:regret_bound}
R(T)
\leq 
\underbrace{\frac{1}{2\eta_1} \|\theta_1 - \theta_1^*\|^2}_{\text{initial error}}
+ 
\underbrace{\frac{G^2}{2} \sum_{t=1}^T \eta_t}_{\text{noise term}}
+ 
\underbrace{\delta \sum_{t=1}^T \frac{1}{\eta_t}}_{\text{drift term}}
+
\underbrace{ G\delta T - \frac{\delta^2 T}{2}}_{\text{Drift-Noise Term}}
\end{equation}

This result reveals a trade-off in designing the effective learning rate $\eta_t = \alpha \cdot \eta_t^{\mathrm{eff}}$, where $\alpha$ is the base learning rate (typically fixed or updated once per epoch), and $\eta_t^{\mathrm{eff}}$, known as the \emph{step size modulation term}, denotes the optimizer’s rescaling factor independent of gradient magnitudes (e.g., bias correction). This term governs the balance between the noise and drift terms in the regret bound.

Choosing $\eta_t^{\mathrm{eff}} > 1$ tends to reduce the impact of dynamic variation (captured by the drift term), while $\eta_t^{\mathrm{eff}} < 1$ helps suppress the noise term. The value $\eta_t^{\mathrm{eff}} \approx 1$ serves as an intuitive reference point between the two effects.

\end{theorem}

\begin{proof}
See Appendix~B.
\end{proof}

Empirically, we observe that after several epochs, the loss decreases more slowly, indicating that Adam often enters a stable regime where $m_t$ and $v_t$ evolve gradually and the effect of gradient noise fades. In contrast, the non-stationarity in time series tasks stems from intrinsic data drift and persists throughout training. To formalize this, we assume that: (1) after a brief initial phase, the moment estimates $m_t$ and $v_t$ evolve slowly over time; and (2) the drift from the time-varying objective is persistent.

Adam’s update can be interpreted as the product of the initial learning rate $\alpha$, a step size modulation term $\eta_t^{\mathrm{eff}} = \frac{\sqrt{1 - \beta_2^t}}{1 - \beta_1^t}$, and a gradient adaptation term $\frac{m_t}{\sqrt{v_t}}$, where for simplicity, all the denominators are assumed to be nonzero.

In standard settings (such as $\beta_1 = 0.9$, $\beta_2 = 0.999$), the trajectory of $\eta_t^{\mathrm{eff}}$ is shown in Fig~\ref{fig:step_curve}. In early training, small $\eta_t^{\mathrm{eff}}$ helps reduce the regret of noisy gradients, consistent with the goal of bias correction. However, due to the much slower decay of $\beta_2^t$, the second-order correction dominates longer, causing $\eta_t^{\mathrm{eff}}$ to remain much lower than 1 for an extended period. Although this design stabilizes early updates, it also severely limits the responsiveness of the optimizer to continuous drift in time-varying objectives. Consequently, the inability to suppress regret from such drift leads to notable performance degradation in non-stationary settings.

\section{Proposed method}
\subsection{Design of TS\_Adam}
\begin{figure*}[t]
\centering
\begin{minipage}{0.42\textwidth}
    \raggedleft
    \includegraphics[width=\linewidth]{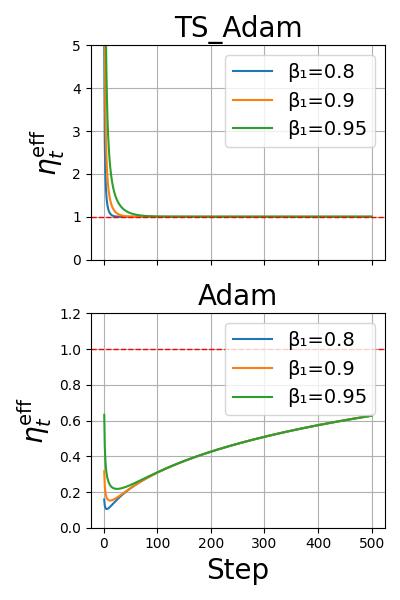}
    \caption{Evolution of the step size modulation term $\eta_t^{\mathrm{eff}}$ with respect to training steps under $\beta_1 \in \{0.8, 0.9, 0.95\}$, for both \textcolor{blue}{TS\_Adam} and \textcolor{red}{Adam}. For visualization purposes, the learning rate $\alpha$ is set to 1. Note that both optimizers asymptotically converge to $\eta_t^{\mathrm{eff}} = \alpha$ during long-term training.}
    \label{fig:step_curve}
\end{minipage}
\hfill
\begin{minipage}{0.56\textwidth}
    \raggedright
    \begin{algorithm}[H]
           \caption{Pseudocode for Time Series Forecasting with \textcolor{red}{Adam} and \textcolor{blue}{TS\_Adam}.}
           \label{algo_adam}
    \begin{algorithmic}[1]
    \REQUIRE Learning rate $\alpha$, exponential decay rates $\beta_1, \beta_2 \in [0,1)$, small constant $\epsilon > 0$
    \REQUIRE Initial parameter vector $\theta_0$
    \STATE Initialize $m_0 = 0, v_0 = 0, t = 0$
    \WHILE{not converged}
        \STATE $t \leftarrow t + 1$
        \STATE $g_t \leftarrow \nabla_\theta f_t(\theta_{t-1})$
        \STATE $m_t \leftarrow \beta_1 m_{t-1} + (1-\beta_1) g_t$
        \STATE $v_t \leftarrow \beta_2 v_{t-1} + (1-\beta_2) g_t^2$
        \STATE $\hat{m}_t \leftarrow m_t / (1-\beta_1^t)$
        \IF{\textcolor{red}{Adam}}
            \STATE $\hat{v}_t \leftarrow v_t / (1-\beta_2^t)$ \COMMENT{Second-order bias correction}
        \ELSE[\textcolor{blue}{TS\_Adam}]
            \STATE $\hat{v}_t \leftarrow v_t$ \COMMENT{No second-order bias correction}
        \ENDIF
        \STATE $\theta_t \leftarrow \theta_{t-1} - \alpha \hat{m}_t / (\sqrt{\hat{v}_t} + \epsilon)$
    \ENDWHILE
    \RETURN $\theta_t$
    \end{algorithmic}
    \end{algorithm}
\end{minipage}
\end{figure*}
To address this issue, we suggest removing the second-order bias correction, allowing $\eta_t^{\mathrm{eff}}$ to remain above 1 during training. Due to the rapid decay of the first-order correction, $\eta_t^{\mathrm{eff}}$ quickly approaches 1, as shown in Fig~\ref{fig:step_curve}, thus maintaining control over the regret induced by gradient noise. In addition, eliminating the second-order correction reduces the computational overhead per iteration, further enhancing the efficiency of the optimizer. Although this adjustment may slightly reduce update stability in the early phase, empirical evidence shows that such an impact is minimal and that the resulting performance gains in non-stationary settings are significant. The pseudocode of TS\_Adam is provided in Algorithm\ref{algo_adam}.

\subsection{Computational and Memory Overhead Analysis.}
As shown in Algorithm~\ref{algo_adam}, TS\_Adam differs from Adam solely by omitting the second-order bias correction (i.e., using $\hat{v}_t \leftarrow v_t$ instead of $\hat{v}_t \leftarrow v_t/(1-\beta_2^t)$). This modification reduces the per-step computational cost by eliminating $n$ division operations, where $n$ is the number of model parameters. Concretely, for a model with $n$ parameters, the approximate floating-point operations (FLOPs) per step are:
\begin{itemize}
\item Adam: $\sim 12n$ FLOPs (comprising $6n$ multiplications, $3n$ additions, $2n$ divisions, and $n$ square roots).
\item TS\_Adam: $\sim 11n$ FLOPs (comprising $6n$ multiplications, $3n$ additions, $n$ divisions, and $n$ square roots).
\end{itemize}
Thus, TS\_Adam achieves a reduction of approximately $8.3\%$ (i.e., $1/12$) in arithmetic operations per training step compared to Adam. In terms of memory consumption, both optimizers maintain identical requirements, needing to store two moment vectors ($m_t$ and $v_t$) of size $n$, which amounts to an extra $2n$ floating-point values (or optimizer states) in addition to the model weights. For a typical model with $n=10^7$ parameters, TS\_Adam saves approximately $10^7$ FLOPs per iteration without introducing any additional memory overhead.

\subsection{Convergence analysis}

The update rule of TS\_Adam only involves the removal of the second-order bias correction and always satisfies the condition $0 \leq \beta_1 \leq \beta_2 < 1$. This adjustment does not violate the assumptions made in the recent convergence analysis of Adam \cite{defossez2020simple, xiao2024adam, li2023convergence}. Therefore, TS\_Adam can provide the following convergence guaranty:

\begin{align} \label{convergence proof}
\mathbb{E}\left[\|\nabla F(x_r)\|^2\right] \leq O\left(\frac{d \ln(N)}{\sqrt{N}}\right)
\end{align}

where \( d \) is the dimension of the problem and \( N \) is the number of iterations. For a detailed proof, readers may refer to the work by ~\cite{defossez2020simple}.

\section{Experiments}
\label{sec:Experiments}
\subsection{Experimental Setup}




To evaluate the effectiveness and generalizability of TS\_Adam, we conduct comparative experiments against several widely used optimizers, including Adam~\cite{kingma2014adam}, AdamW~\cite{loshchilov2017decoupled}, Lookahead~\cite{zhang2019lookahead}, SGD, and Yogi~\cite{zaheer2018adaptive}.

We consider both long-term and short-term forecasting tasks. For long-term forecasting, we use three representative model architectures: MICN, PatchTST, and SegRNN, on the ETT benchmark, including ETTh1, ETTh2, ETTm1, ETTm2, ECL, and Weather datasets. Each model is evaluated in four forecast horizons: 96, 192, 336, and 720 time steps. For short-term forecasting, the same models are applied to the M4 dataset using the official evaluation setup: 48 steps for hourly, 14 for daily, 8 for weekly, 6 for monthly, 5 for quarterly and 2 for yearly series.


All experiments have been conducted on a single NVIDIA RTX 3090 GPU. For the long-term forecasting tasks, we follow the common practice in time series forecasting literature for dataset splitting: the ETT dataset is chronologically divided into 60\% for training, 20\% for validation, and 20\% for testing; the ECL and Weather datasets follow a split of 70\% for training, 20\% for validation, and 10\% for testing. For the short-term forecasting task on the M4 dataset, we strictly adhere to the official train/test split provided by the M4 competition, which is the standard benchmark setting.

To ensure reproducibility, we fix the random-seed sequence to \{123, 2021, 2077\} and run each configuration three times. The reported results are the mean and standard deviation of evaluation metrics (including MSE, MAE, and sMAPE) computed on the test sets over these three independent runs.

We provide complete implementation details to ensure reproducibility. The anonymized code repository contains complete data preprocessing scripts, environment configuration files, the full list of fixed random seeds used across all experiments, a comprehensive evaluation configuration table documenting all hyperparameters, and one-click execution scripts with documented example usage. All required datasets are publicly available.

\subsection{Performance Comparison}

\subsubsection{Long-Term Forecasting}


The ETT (Electricity Transformer Temperature) dataset, along with ECL (Electricity Consumption Load) and Weather, is widely used for time series forecasting benchmarking. Collected from real-world systems, they exhibit periodicity, long-term trends, and distributional shifts. ETT includes hourly (ETTh) and 15-minute (ETTm) subsets from an electricity grid. The ECL records the hourly power consumption, while Weather provides 10-minute meteorological data.
﻿
Due to their practical relevance and temporal complexity, these datasets have been commonly adopted in recent forecasting studies~\cite{nie2022time, liu2023itransformer, wang2024timexer, lin2023segrnn}. All baselines are reproduced under identical settings to ensure fair comparison.

\begin{landscape}  
\begin{table*}[ht]
  \centering
  \caption{
    Performance comparison of different optimizers on long-term forecasting using the four subsets of the ETT dataset, the Weather dataset, and the ECL dataset. 
    Results are averaged over four prediction lengths: \{96, 192, 336, 720\}. Each entry reports the mean and standard deviation across three runs. The best result in each row is marked in bold red, and the second-best is underlined in blue.
    }

  \label{tab:long_time_comparison}
  \resizebox{1.3\textwidth}{!}{
    \begin{tabular}{c|c|ccc|ccc|ccc}
    \toprule

    \multirow{2}{*}{\rotatebox{90}{\footnotesize\textbf{Dataset}}} & 
    \multirow{2}{*}{\footnotesize\textbf{Method}} & 
    \multicolumn{3}{c}{\makecell{\footnotesize Model\_1\\\footnotesize(MICN)}} & 
    \multicolumn{3}{c}{\makecell{\footnotesize Model\_2\\\footnotesize(PatchTST)}} & 
    \multicolumn{3}{c}{\makecell{\footnotesize Model\_3\\\footnotesize(SegRNN)}} \\
    \cmidrule(lr){3-5} \cmidrule(lr){6-8} \cmidrule(l){9-11}
    & & \textbf{\footnotesize MSE} & \textbf{\footnotesize MAE} & \textbf{\footnotesize SMAPE} & \textbf{\footnotesize MSE} & \textbf{\footnotesize MAE} & \textbf{\footnotesize SMAPE} & \textbf{\footnotesize MSE} & \textbf{\footnotesize MAE} & \textbf{\footnotesize SMAPE} \\

    \midrule
    \multirow{6}[2]{*}{\rotatebox{90}{\footnotesize ETTh1}} & TS\_Adam & \textcolor{red}{0.441±0.005} & \textcolor{red}{0.460±0.003} & \textcolor{red}{84.304 ± 0.035} & \textcolor{red}{0.447±0.003} & \textcolor{red}{0.439±0.001} & \textcolor{red}{80.274 ± 0.02} & \textcolor{red}{0.440±0.002} & \textcolor{red}{0.443±0.001} & \textcolor{red}{80.637 ± 0.004} \\
          & Adam  & 0.569±0.023 & 0.528±0.01 & 92.79 ± 0.031 & 0.457±0.002 & 0.445±0.001 & 81.448 ± 0.055 & 0.473±0.003 & 0.462±0.002 & {\textcolor{blue}{\uline{83.604 ± 0.079}}} \\
          & SGD   & 0.940±0.007 & 0.714±0.004 & 127.224 ± 0.062 & 0.899±0.014 & 0.644±0.005 & 104.293 ± 0.001 & 1.330±0.004 & 0.796±0.07 & 106.171 ± 0.054 \\
          & Yogi  & {\textcolor{blue}{\uline{0.457±0.002}}} & {\textcolor{blue}{\uline{0.476±0.001}}} & {\textcolor{blue}{\uline{85.541 ± 0.066}}} & 0.511±0.002 & 0.479±0.001 & 86.056 ± 0.02 & 0.652±0.003 & 0.540±0.001 & 95.296 ± 0.056 \\
          & AdamW & 0.568±0.023 & 0.528±0.01 & 92.785 ± 0.076 & {\textcolor{blue}{\uline{0.457±0.002}}} & {\textcolor{blue}{\uline{0.445±0.001}}} & {\textcolor{blue}{\uline{81.43 ± 0.006}}} & {\textcolor{blue}{\uline{0.473±0.003}}} & {\textcolor{blue}{\uline{0.462±0.002}}} & 83.625 ± 0.006 \\
          & Lookahead & 0.534±0.004 & 0.507±0.002 & 89.516 ± 0.07 & 0.465±0.002 & 0.45±0.002 & 81.751 ± 0.025 & 0.542±0.003 & 0.494±0.001 & 88.836 ± 0.024 \\
    \midrule
    \multirow{6}[2]{*}{\rotatebox{90}{\footnotesize ETTh2}} & TS\_Adam & \textcolor{red}{0.533±0.01} & \textcolor{red}{0.501±0.005} & \textcolor{red}{68.554 ± 0.046} & \textcolor{red}{0.373±0.002} & \textcolor{red}{0.399±0.001} & \textcolor{red}{60.082 ± 0.03} & \textcolor{red}{0.377±0.005} & \textcolor{red}{0.409±0.002} & \textcolor{red}{60.728 ± 0.088} \\
          & Adam  & 0.581±0.027 & 0.526±0.014 & 71.621 ± 0.089 & 0.378±0.005 & 0.403±0.001 & 60.473 ± 0.017 & 0.384±0.002 & 0.411±0.001 & {\textcolor{blue}{\uline{61.808 ± 0.027}}} \\
          & SGD   & 1.480±0.018 & 0.889±0.006 & 109.109 ± 0.047 & 0.469±0.009 & 0.463±0.001 & 67.796 ± 0.02 & 0.557±0.002 & 0.511±0.002 & 72.308 ± 0.07 \\
          & Yogi  & {\textcolor{blue}{\uline{0.572±0.002}}} & {\textcolor{blue}{\uline{0.522±0.001}}} & 70.715 ± 0.021 & 0.405±0.003 & 0.421±0.001 & 63.043 ± 0.043 & 0.438±0.002 & 0.445±0.001 & 64.423 ± 0.088 \\
          & AdamW & 0.581±0.027 & 0.526±0.014 & 71.576 ± 0.002 & {\textcolor{blue}{\uline{0.380±0.001}}} & {\textcolor{blue}{\uline{0.403±0.001}}} & {\textcolor{blue}{\uline{60.458 ± 0.078}}} & {\textcolor{blue}{\uline{0.384±0.002}}} & {\textcolor{blue}{\uline{0.411±0.001}}} & 61.835 ± 0.062 \\
          & Lookahead & 0.575±0.014 & 0.523±0.006 & {\textcolor{blue}{\uline{70.552 ± 0.018}}} & 0.386±0.001 & 0.407±4.3e-4 & 60.955 ± 0.065 & 0.395±0.001 & 0.418±0.001 & 63.44 ± 0.016 \\
    \midrule
    \multirow{6}[2]{*}{\rotatebox{90}{\footnotesize ETTm1}} & TS\_Adam & \textcolor{red}{0.382±0.004} & \textcolor{red}{0.411±0.003} & {\textcolor{blue}{\uline{80.193 ± 0.075}}} & \textcolor{red}{0.393±0.001} & \textcolor{red}{0.399±0.001} & 75.717 ± 0.06 & \textcolor{red}{0.447±0.003} & \textcolor{red}{0.440±0.001} & \textcolor{red}{81.772 ± 0.073} \\
          & Adam  & 0.407±0.011 & 0.432±0.008 & 83.561 ± 0.079 & 0.397±0.001 & 0.401±0.001 & {\textcolor{blue}{\uline{75.395 ± 0.013}}} & 0.470±0.003 & 0.449±0.001 & 83.279 ± 0.024 \\
          & SGD   & 0.833±0.002 & 0.654±0.001 & 125.346 ± 0.068 & 0.816±0.034 & 0.602±0.003 & 99.61 ± 0.033 & 1.270±0.004 & 0.708±0.002 & 102.375 ± 0.085 \\
          & Yogi  & {\textcolor{blue}{\uline{0.393±0.001}}} & {\textcolor{blue}{\uline{0.412±0.001}}} & \textcolor{red}{78.41 ± 0.094} & 0.424±0.001 & 0.412±0.001 & 76.807 ± 0.078 & 0.679±0.001 & 0.542±0.001 & 97.371 ± 0.029 \\
          & AdamW & 0.407±0.011 & 0.432±0.008 & 83.561 ± 0.029 & {\textcolor{blue}{\uline{0.397±8.5e-4}}} & {\textcolor{blue}{\uline{0.401±0.001}}} & 75.391 ± 0.022 & {\textcolor{blue}{\uline{0.470±0.003}}} & {\textcolor{blue}{\uline{0.449±0.001}}} & {\textcolor{blue}{\uline{83.267 ± 0.089}}} \\
          & Lookahead & 0.399±0.004 & 0.427±0.003 & 82.733 ± 0.064 & 0.405±0.001 & 0.401±3.9e-4 & \textcolor{red}{75.076 ± 0.046} & 0.547±0.006 & 0.482±0.003 & 88.423 ± 0.026 \\
    \midrule
    \multirow{6}[2]{*}{\rotatebox{90}{\footnotesize ETTm2}} & TS\_Adam & \textcolor{red}{0.346±0.009} & \textcolor{red}{0.394±0.006} & {\textcolor{blue}{\uline{57.321 ± 0.068}}} & \textcolor{red}{0.282±4.8e-4} & {\textcolor{blue}{\uline{0.328±0.001}}} & \textcolor{red}{51.248 ± 0.1} & \textcolor{red}{0.288±0.001} & \textcolor{red}{0.337±0.001} & \textcolor{red}{53.195 ± 0.028} \\
          & Adam  & 0.352±0.009 & 0.395±0.006 & 57.435 ± 0.032 & 0.283±4.2e-4 & 0.328±3.4e-4 & 51.392 ± 0.008 & 0.291±0.001 & 0.339±0.001 & 53.742 ± 0.068 \\
          & SGD   & 1.190±0.006 & 0.772±0.002 & 98.354 ± 0.099 & 0.353±0.001 & 0.381±0.001 & 58.092 ± 0.038 & 0.402±0.002 & 0.427±0.002 & 64.801 ± 0.096 \\
          & Yogi  & 0.366±0.005 & 0.411±0.003 & 59.16 ± 0.091 & 0.293±3.5e-4 & 0.335±2.9e-4 & 52.24 ± 0.043 & 0.302±2.2e-4 & 0.348±2.6e-4 & 55.729 ± 0.037 \\
          & AdamW & {\textcolor{blue}{\uline{0.352±0.009}}} & {\textcolor{blue}{\uline{0.395±0.006}}} & 57.432 ± 0.044 & {\textcolor{blue}{\uline{0.283±4.1e-4}}} & \textcolor{red}{0.328±3.3e-4} & {\textcolor{blue}{\uline{51.347 ± 0.052}}} & {\textcolor{blue}{\uline{0.291±9.3e-4}}} & {\textcolor{blue}{\uline{0.339±0.001}}} & {\textcolor{blue}{\uline{53.725 ± 0.038}}} \\
          & Lookahead & 0.354±0.006 & 0.396±0.004 & \textcolor{red}{57.189 ± 0.051} & 0.284±3.4e-4 & 0.329±2.6e-4 & 51.581 ± 0.077 & 0.296±2.4e-4 & 0.343±4.3e-4 & 54.632 ± 0.001 \\
    \midrule
    \multirow{6}[2]{*}{\rotatebox{90}{\footnotesize ECL}} & TS\_Adam & \textcolor{red}{0.194±5.8e-4} & \textcolor{red}{0.303±1.7e-4} & \textcolor{red}{57.33 ± 0.026} & {\textcolor{blue}{\uline{0.222±2.1e-4}}} & {\textcolor{blue}{\uline{0.309±8.2e-4}}} & 56.945 ± 0.024 & \textcolor{red}{0.214±3.1e-4} & \textcolor{red}{0.299±4.1e-4} & \textcolor{red}{56.34 ± 0.053} \\
          & Adam  & 0.195±0.003 & 0.307±0.001 & {\textcolor{blue}{\uline{58.115 ± 0.002}}} & 0.225±0.001 & 0.312±0.001 & \textcolor{red}{56.574 ± 0.021} & {\textcolor{blue}{\uline{0.227±7.3e-4}}} & {\textcolor{blue}{\uline{0.309±2.6e-4}}} & {\textcolor{blue}{\uline{57.78 ± 0.017}}} \\
          & SGD   & 1.218±0.009 & 0.893±0.001 & 140.115 ± 0.081 & 0.586±0.003 & 0.581±0.001 & 102.234 ± 0.071 & 1.512±0.003 & 0.936±0.002 & 123.649 ± 0.014 \\
          & Yogi  & 0.301±4.2e-4 & 0.397±1.2e-4 & 69.599 ± 0.062 & 0.223±0.001 & 0.308±8.7e-4 & 57.981 ± 0.075 & 0.288±3.5e-4 & 0.369±1.6e-4 & 67.208 ± 0.09 \\
          & AdamW & {\textcolor{blue}{\uline{0.195±0.003}}} & {\textcolor{blue}{\uline{0.307±8.7e-4}}} & 58.143 ± 0.006 & 0.225±8.4e-4 & 0.312±4.6e-4 & {\textcolor{blue}{\uline{56.586 ± 0.02}}} & 0.227±7.1e-4 & 0.309±3.5e-4 & 57.79 ± 0.037 \\
          & Lookahead & 0.229±2.8e-4 & 0.334±3.8e-4 & 62.176 ± 0.073 & \textcolor{red}{0.213±6.2e-4} & \textcolor{red}{0.301±2.9e-4} & 56.911 ± 0.079 & 0.238±2.4e-4 & 0.320±3.4e-4 & 59.552 ± 0.065 \\
    \midrule
    \multirow{6}[2]{*}{\rotatebox{90}{\footnotesize Weather}} & TS\_Adam & \textcolor{red}{0.263±1.9e-4} & \textcolor{red}{0.311±7.1e-4} & \textcolor{red}{76.083 ± 0.016} & \textcolor{red}{0.254±4.7e-4} & \textcolor{red}{0.275±4.9e-4} & \textcolor{red}{63.057 ± 0.092} & \textcolor{red}{0.251±4.4e-4} & \textcolor{red}{0.297±7e-4} & 80.688 ± 0.097 \\
          & Adam  & 0.268±5.8e-4 & 0.319±0.001 & 77.613 ± 0.047 & 0.257±2.5e-4 & 0.279±1.3e-4 & 64.133 ± 0.041 & 0.252±3.8e-4 & 0.299±5.6e-4 & 81.671 ± 0.045 \\
          & SGD   & 0.604±0.01 & 0.592±0.008 & 125.428 ± 0.089 & 0.302±3.4e-4 & 0.319±2.0e-4 & 69.494 ± 0.053 & 0.369±9.7e-4 & 0.359±0.001 & 81 ± 0.002 \\
          & Yogi  & 0.272±2.5e-4 & 0.328±4.9e-4 & 81.87 ± 0.068 & 0.270±0.001 & 0.287±9.4e-4 & 64.429 ± 0.033 & 0.274±4.9e-4 & 0.318±6.8e-4 & \textcolor{red}{76.61 ± 0.044} \\
          & AdamW & {\textcolor{blue}{\uline{0.268±2.7e-4}}} & {\textcolor{blue}{\uline{0.319±0.001}}} & {\textcolor{blue}{\uline{77.59 ± 0.043}}} & {\textcolor{blue}{\uline{0.257±2.1e-4}}} & {\textcolor{blue}{\uline{0.279±1e-4}}} & 64.587 ± 0.073 & {\textcolor{blue}{\uline{0.252±3.8e-4}}} & {\textcolor{blue}{\uline{0.299±4.7e-4}}} & 81.679 ± 0.014 \\
          & Lookahead & 0.270±2.2e-4 & 0.324±4.9e-4 & 79.114 ± 0.027 & 0.259±4.2e-4 & 0.281±2.3e-4 & {\textcolor{blue}{\uline{64.115 ± 0.057}}} & 0.254±1.6e-4 & 0.300±1.2e-4 & {\textcolor{blue}{\uline{80.407 ± 0.024}}} \\
    \bottomrule
    \end{tabular}%
    }
\end{table*}%
\end{landscape}  

\begin{table}[h]
  \centering
  \caption{Variance contributions from STL decomposition of the ETT, ECL and Weather datasets. Note: contributions may not sum to one due to non-orthogonality.}
  
  \label{tab:mer base}
  \resizebox{0.7\textwidth}{!}{
    \begin{tabular}{c|cccccc}
    \toprule
          & ETTh1 & ETTh2 & ETTm1 & ETTm2 & ECL & Weather \\
    \midrule
    Trend & 59.3\% & 88.9\% & 59.1\% & 88.5\% & 20.9\% & 51.0\%\\
    Seasonal & 32.1\% & 5.0\% & 31.9\% & 5.0\% & 73.3\% & 27.4\%\\
    Residual & 15.5\% & 6.8\% & 15.5\% & 6.8\% & 11.4\% & 28.3\%\\
    \bottomrule
    \end{tabular}%
    }
\end{table}%



Table~\ref{tab:long_time_comparison} reports the long-term forecast results across the ETT, ECL, and Weather benchmarks. TS\_Adam consistently achieves lower prediction errors than all baselines in most model–dataset combinations. On the four datasets with comparable volumes (ETTm1, ETTm2, ECL, and Weather), it yields an average relative reduction over Adam of 3.6\% in MSE and 2.2\% in MAE. The reductions are most pronounced on ETTm1 (4.0\% MSE, 2.5\% MAE) and ECL (3.1\% MSE, 2.2\% MAE), and more modest on ETTm2 (1.0\% MSE, 0.3\% MAE) and Weather (1.2\% MSE, 1.5\% MAE).

The performance improvement is even more substantial on the hourly-sampled ETT datasets (ETTh1 and ETTh2), which contain fewer observations. In ETTh1, TS\_Adam provides a remarkable relative reduction of 10.5\% in MSE and 6.1\% in MAE compared to Adam, while in ETTh2 the improvements are 3.8\% in MSE and 2.1\% in MAE. Although the absolute differences may appear modest, they are substantial given that the underlying forecasting models already represent state-of-the-art architectures with strong generalization capabilities.

Furthermore, to provide a scale-independent evaluation of relative forecast errors, we report the Symmetric Mean Absolute Percentage Error (SMAPE). As presented in the Table~\ref{tab:long_time_comparison}, TS\_Adam achieves an average reduction in SMAPE of 1.5\% compared to Adam across the same datasets, with the most pronounced improvements observed on ETT (2.3\%) and Weather (1.3\%). This consistent superiority across MSE, MAE, and SMAPE—metrics that emphasize squared errors, absolute errors, and relative percentage errors, respectively—demonstrates that the performance gains of TS\_Adam are robust and not dependent on any particular error measure. These results underscore the effectiveness of TS\_Adam in enhancing overall forecast accuracy.


To explain the observed performance patterns and quantify non-stationarity, we leverage the STL decomposition results (Table~\ref{tab:mer base}) alongside our theoretical framework. Among the four larger-scale datasets (ETTm1, ETTm2, ECL, and Weather), ECL and ETTm1 show the strongest seasonal contributions (73.3\% and 31.9\%, respectively) coupled with lower residual components (11.4\% and 15.5\%). This structure aligns with the greater performance gains of TS\_Adam on these datasets, consistent with our theoretical analysis: Theorem~1 posits that the seasonal component induces rapid distributional shifts, and Theorem~2 indicates that such faster drift ($\delta$) amplifies dynamic regret. As TS\_Adam is designed to suppress regret accumulation under distributional drift while relaxing the suppression of stochastic noise, its effectiveness is most pronounced when seasonal strength is prominent and the residual proportion is small—thereby establishing a quantifiable link between optimizer performance and non-stationarity, characterized by seasonal and residual variance.

The particularly strong gains on ETTh1 and ETTh2 are attributable to their smaller dataset sizes (hourly sampling over shorter durations), where limited training data amplifies the advantage of an optimizer adept at adapting to non-stationary patterns. This contrast further underscores that, under sufficient data volume, the superiority of TS\_Adam is meaningfully correlated with the seasonal variance metric, validating its design principle.

\begin{table}[h]
  \centering
  \caption{Pairwise comparison p-values (MSE / MAE) between TS\_Adam and baselines. Asterisks indicate significance in pairwise t-tests with Bonferroni correction ($p < 0.05$). The "Sig.Wins" column reports the number of datasets (out of 6) where TS\_Adam shows statistically superior performance over the corresponding baseline for MSE (before the slash) and MAE (after the slash), respectively. For example, "3/3" means TS\_Adam significantly outperforms the baseline on 3 datasets according to MSE and on 3 datasets according to MAE. {Note: LA = Lookahead.}}
  
  \label{tab:t_test}
  \resizebox{1\textwidth}{!}{
  \setlength{\tabcolsep}{0.1pt}
  \renewcommand{\arraystretch}{1.3}  
  \huge
   \begin{tabular}{c|cccccc|c}
    \toprule
    Comp. & ETTh1 & ETTh2 & ETTm1 & ETTm2 & ECL   & Weahter & Sig.Wins \\
    \midrule
    vs Adam & {0.0142/0.0129} & 0.025/0.0342 & 0.0013*/0.0094* & 0.011/0.0752 & 0.0039*/<0.0001* & 0.0008*/0.0011* & 3/3 \\
    vs SGD & <0.0001*/<0.0001* & 0.017/0.007* & <0.0001*/<0.0001* & 0.0253/0.0095 & 0.0002*/<0.0001* & 0.0051*/0.0096 & 4/4 \\
    vs Yogi & 0.0109/0.0027* & <0.0001*/<0.0001* & 0.0326/0.0416 & <0.0001*/<0.0001* & 0.0047*/0.0051* & <0.0001*/<0.0001* & 4/5 \\
    vs AdamW & 0.0141/0.0129 & 0.0189/0.0346 & 0.0013*/0.0093 & 0.011/0.0761 & 0.004*/<0.0001* & 0.0008*/0.0011* & 3/2 \\
    vs LA & 0.0007*/0.0005* & 0.001*/0.0011* & 0.0174/0.0104 & 0.0025*/0.0194 & 0.0358/0.0332 & 0.0001*/0.001* & 4/3 \\
    \bottomrule
    \end{tabular}%
    }
\end{table}%

The statistical significance Table~\ref{tab:t_test} confirm that the performance improvements of TS\_Adam are not due to random chance. TS\_Adam demonstrates statistically significant superiority over all baseline optimizers across the majority of datasets and evaluation metrics, as indicated by the high number of significant wins (often 3 or more out of 6 datasets per comparator) after strict Bonferroni correction for multiple comparisons. This provides strong statistical evidence for the robustness and reliability of the observed experimental results.

\subsubsection{Short-Term Forecasting}
To evaluate optimizer performance in short-term forecasting scenarios, we use the M4 dataset, a large-scale and diverse benchmark comprising 100{,}000 real-world time series from domains such as finance, economics, industry, and demography. The series are categorized into six frequency groups: yearly, quarterly, monthly, weekly, daily, and hourly, each associated with predefined forecast horizons. This setting supports robust evaluation across a range of time scales.

The M4 dataset, with its broad domain coverage and industrial relevance, has become a standard benchmark for short-term forecasting. and recent state-of-the-art models report results on this dataset~\cite{wu2022timesnet}. Its diversity and scale enable comprehensive evaluation of model generalization and performance.

\begin{table*}[ht]
  \centering
  \caption{
    Performance comparison of different optimizers on short-term forecasting using the M4 dataset. Results are reported on six frequency-based subsets following the official evaluation protocol of the M4 competition. Each entry presents the mean and standard deviation of forecasting errors. The best result in each row is marked in bold red, and the second-best is underlined in blue.
    }
  \label{tab:short_term}
  \resizebox{1\textwidth}{0.65\height}{
    \begin{tabular}{c|ccc|ccc|ccc}
    \toprule
    \multirow{2}[4]{*}{{Method}} & \multicolumn{3}{c|}{Model\_1(MICN)} & \multicolumn{3}{c|}{Model\_2(PatchTST)} & \multicolumn{3}{c}{model\_3(SegRNN)} \\
\cmidrule{2-10}          & \multicolumn{1}{c|}{SMAPE} & \multicolumn{1}{c|}{MASE} & OWA   & \multicolumn{1}{c|}{SMAPE} & \multicolumn{1}{c|}{MASE} & OWA   & \multicolumn{1}{c|}{SMAPE} & \multicolumn{1}{c|}{MASE} & OWA \\
    \midrule
    TS\_Adam & \textcolor{blue}{\uline{15.451±0.017}} & {\textcolor{red}{2.528±0.004}} & {\textcolor{red}{1.231±0.002}} & {\textcolor{red}{12.813±0.028}} & {\textcolor{red}{1.780±0.009}} & {\textcolor{red}{0.937±0.003}} & {\textcolor{red}{13.960±0.003}} & {\textcolor{red}{1.967±0.001}} & {\textcolor{red}{1.029±0.001}} \\
    Adam  & 16.524±0.006 & 3.167±0.001 & 1.437±0.001 & \textcolor{blue}{\uline{13.943±0.071}} & \textcolor{blue}{\uline{2.092±0.027}} & \textcolor{blue}{\uline{1.061±0.01}} & \textcolor{blue}{\uline{14.021±0.003}} & \textcolor{blue}{\uline{1.997±0.001}} & \textcolor{blue}{\uline{1.039±0.001}} \\
    SGD   & {\textcolor{red}{15.165±0.008}} & \textcolor{blue}{\uline{2.637±0.005}} & \textcolor{blue}{\uline{1.248±0.001}} & 14.911±0.062 & 2.322±0.027 & 1.157±0.009 & 14.208±0.001 & 2.044±0.001 & 1.058±0.001 \\
    Yogi  & 16.913±0.006 & 3.263±0.002 & 1.477±0.001 & 14.319±0.075 & 2.178±0.025 & 1.097±0.009 & 14.171±0.001 & 2.031±0.001 & 1.053±0.001 \\
    AdamW & 16.524±0.006 & 3.166±0.002 & 1.437±0.001 & 13.946±0.072 & 2.094±0.027 & 1.062±0.01 & 14.023±0.002 & 1.997±0.001 & 1.039±0.001 \\
    Lookahead & 18.748±0.007 & 5.518±0.007 & 2.134±0.002 & 15.415±0.171 & 2.444±0.032 & 1.207±0.015 & 14.115±0.001 & 2.018±0.001 & 1.048±0.001 \\
    \bottomrule
    \end{tabular}%
    }
\end{table*}%

We adopt the official evaluation metrics from the M4 competition: symmetric mean absolute percentage error (SMAPE), mean absolute scaled error (MASE), and overall weighted average (OWA). As shown in Table~\ref{tab:short_term}, TS\_Adam consistently achieves lower errors than Adam across all three metrics. On average, it achieves relative reductions of 5.0\% in SMAPE, 12.2\% in MASE, and 7.1\% in OWA, corresponding to absolute decreases of 0.76, 0.33, and 0.08, respectively. These reductions indicate that TS\_Adam effectively enhances performance in highly diverse and irregular short-term forecasting settings.

This observed improvement is statistically significant, as confirmed by pairwise t-tests comparing TS\_Adam against each baseline across all M4 series. The results, summarized in Table~\ref{tab:t_test_s}, show that TS\_Adam significantly outperforms Adam, Yogi, AdamW, and Lookahead in all three metrics (all $p < 0.05$), while also surpassing SGD in MASE and OWA. These statistical tests validate that performance gains are robust and not attributable to random variation.

\begin{table}[h]
  \centering
  \caption{Statistical significance of TS\_Adam versus baselines on the M4 dataset for short-term forecasting. P-values are from pairwise t-tests comparing forecast accuracy across all series of the M4 dataset.}
  
  \label{tab:t_test_s}
  \resizebox{0.5\textwidth}{!}{
   \begin{tabular}{cccc}
    \toprule
    Comparison Pair & SMAPE & MASE  & OWA \\
    \midrule
    vs Adam & 0.0025 & 0.0059 & 0.0041 \\
    vs SGD & 0.0938 & 0.012 & 0.0267 \\
    vs Yogi & 0.0011 & 0.0034 & 0.0022 \\
    vs AdamW & 0.0024 & 0.0059 & 0.004 \\
    vs Lookahead & 0.0029 & 0.0247 & 0.0165 \\
    \bottomrule
    \end{tabular}%
    }
\end{table}%

\subsubsection{Convergence Behavior and Practical Implications}

To complement the quantitative evaluation, we analyze the convergence behavior of TS\_Adam compared to the standard Adam optimizer. Fig~\ref{fig:loss_function} presents the test loss trajectories for representative long- and short-term forecasting tasks. Consistent with the theoretical analysis in Section~\ref{sec:Impact of Temporal NonStationarity on Optimization}, TS\_Adam exhibits faster convergence and achieves lower final loss values in most cases.

Although minor oscillations occur in the early training phase due to the removal of the second-order bias correction, they remain within acceptable bounds and do not compromise training stability. To systematically assess these aspects, we conduct robustness tests under two challenging scenarios in Section~\ref{sec:Model Robustness to Data and Training Perturbations}: (i) Gaussian noise, which amplifies gradient variance and directly tests the optimizer's sensitivity to the oscillations mentioned above; and (ii) extreme outliers, which simulate abrupt, severe data perturbations to evaluate stability under quasi-nonstationary conditions. The results confirm that TS\_Adam maintains a superior and more stable performance compared to Adam in both settings. Instead, this controlled adjustment improves the responsiveness of the optimizer to distributional changes, ultimately leading to better forecast performance.

From a practical standpoint, such improvements in convergence speed and forecasting performance are especially beneficial in time series tasks that demand both training efficiency and model stability. These results indicate that TS\_Adam achieves a desirable balance between adaptability to non-stationarity and overall robustness, supporting its utility in a broad range of forecasting scenarios.

\begin{figure*}[ht]
\vskip 0.2in
\begin{center}
\includegraphics[width=1\textwidth]{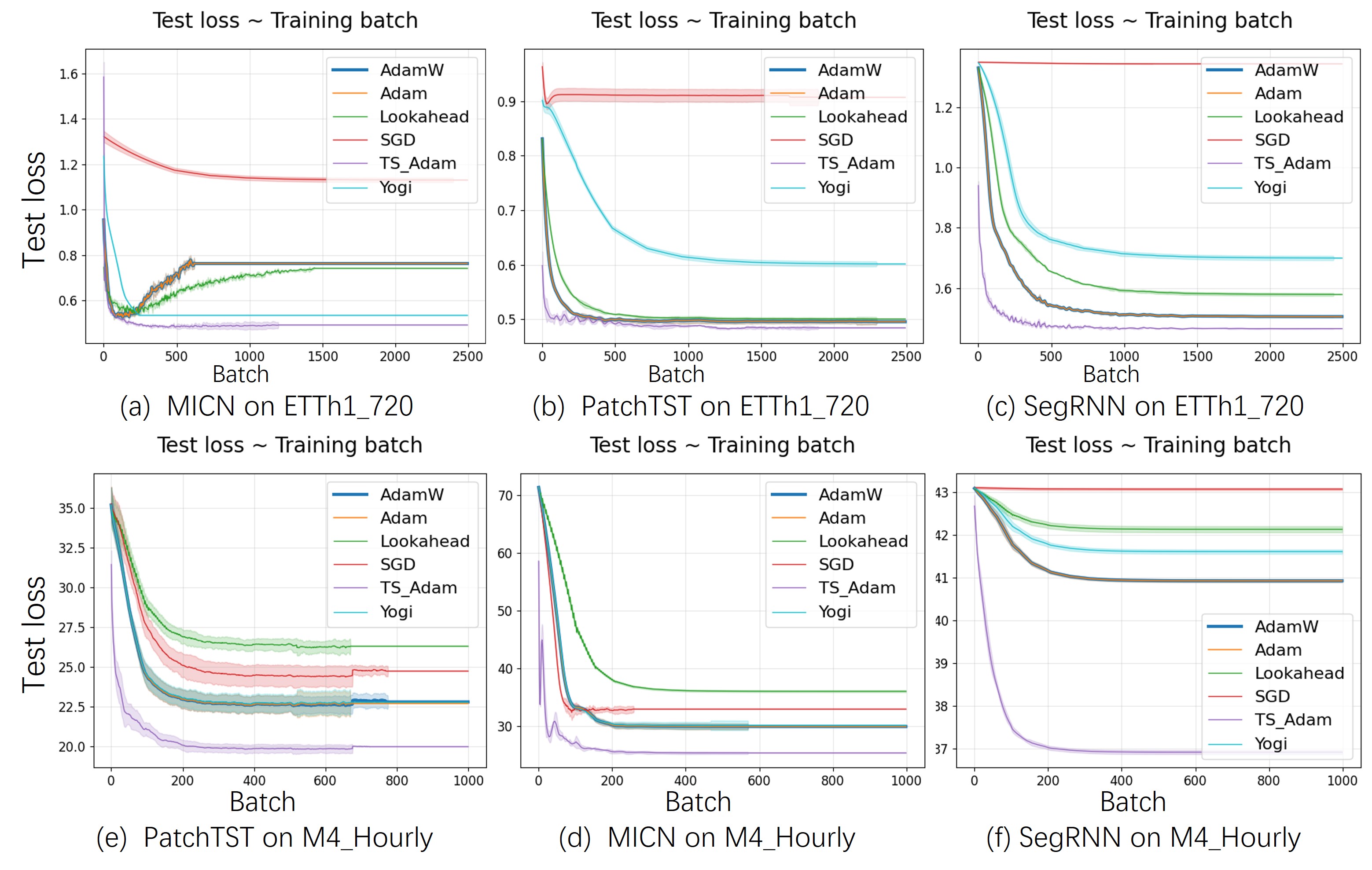}
\caption{Test loss curves during training for various optimizers on ETTh1 (T=720) and M4-Hourly, using MICN, PatchTST, and SegRNN. Each curve shows the mean test loss, with shaded areas indicating one standard deviation. For optimizers that stop early, a flat horizontal line denotes the final value without uncertainty shading.} 
\label{fig:loss_function}
\end{center}
\vskip -0.2in
\end{figure*}

\subsection{Ablation Studies}

We perform ablation studies to evaluate whether the observed improvements stem from core design choices rather than hyperparameter tuning or dataset-specific effects, and to assess the generalizability of TS\_Adam across different optimization scenarios. Experiments are performed on the MICN model using ETTh1 and ETTh2, two widely adopted long-term forecasting benchmarks with distinct seasonal and trend characteristics. This setting allows us to validate the practical reliability and theoretical consistency of TS\_Adam under distributional drift.

\subsubsection{Hyperparameter Sensitivity Analysis}

\paragraph{Sensitivity to internal hyperparameters}

\begin{figure*}[ht]
\vskip 0.2in
\begin{center}
\includegraphics[width=1\textwidth]{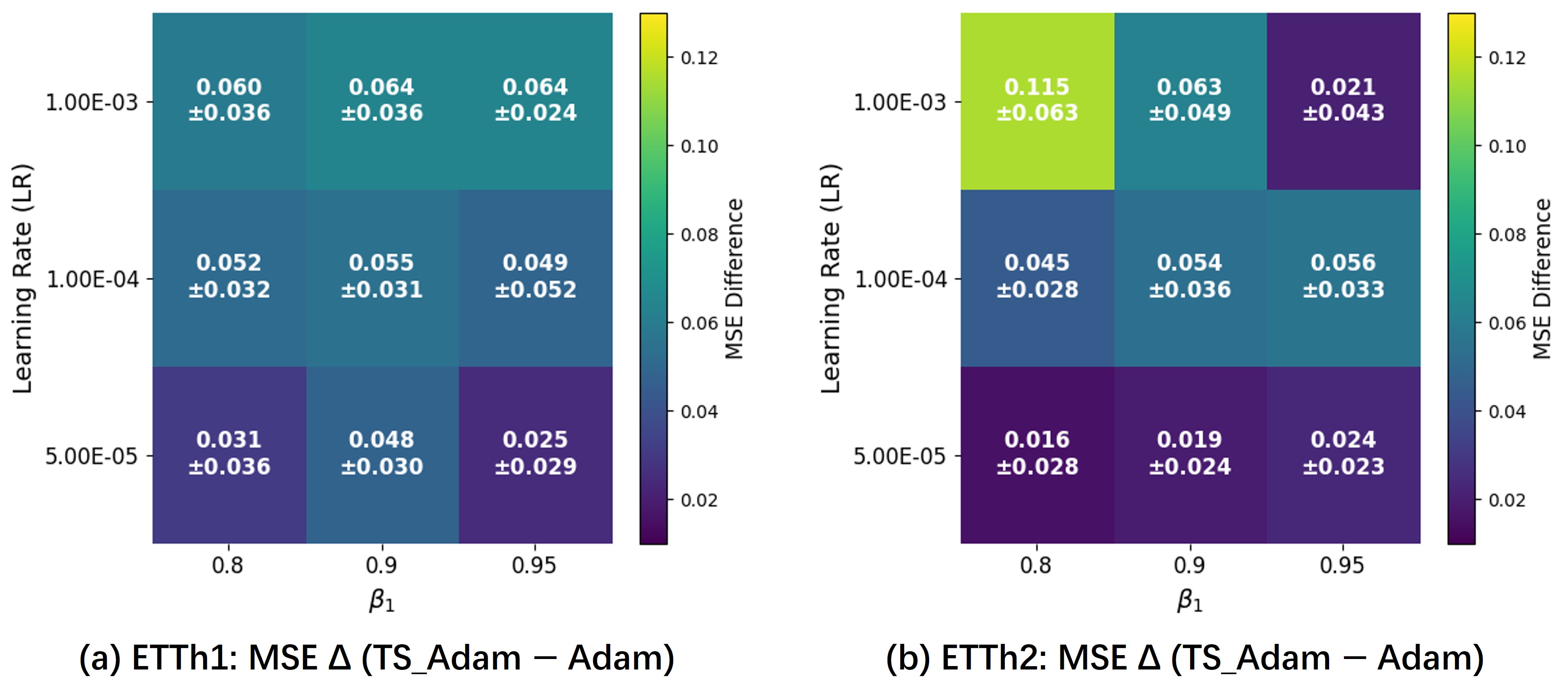}
\caption{Sensitivity of TS\_Adam to $\beta_1$ and learning rate on ETTh1 (a) and ETTh2 (b). Each cell shows the mean ± standard deviation of the absolute MSE difference (TS\_Adam − Adam), averaged over four prediction lengths.} 
\label{fig:hyperparam_sensitivity}
\end{center}
\vskip -0.2in
\end{figure*}

We analyze the sensitivity of TS\_Adam to two key hyperparameters: the learning rate $\alpha$ and the first-order decay coefficient $\beta_1$. This choice is grounded in our theoretical formulation, where TS\_Adam explicitly modifies the effective learning rate in Adam:
\[
\eta_t = \alpha \cdot \frac{\sqrt{1 - \beta_2^t}}{1 - \beta_1^t}.
\]

This term, as discussed in Theorem~2,  directly influences the trade-off between distributional drift and gradient noise in the dynamic regret bound. Since TS\_Adam removes $\beta_2$ by design, we focus on the remaining influential factors($\alpha$ and $\beta_1$) to fairly assess its internal behavior.

Fig~\ref{fig:hyperparam_sensitivity} reports the reduction in MSE of TS\_Adam relative to Adam under various combinations of learning rate $\alpha$ and first-order momentum coefficient $\beta_1$. 

Across both datasets, TS\_Adam consistently outperforms Adam under most configurations, confirming its robustness to internal hyperparameter variations. Performance gain is more sensitive to learning rate: larger $\alpha$ values generally produce greater MSE reductions, with the best improvements occurring when $\alpha = 1\text{e-}3$. This supports our theoretical analysis that higher learning rates magnify regret accumulation differences between TS\_Adam and Adam, leading to larger gains.

In contrast, $\beta_1$ has a less pronounced effect. Although there are slight variations across $\beta_1 \in \{0.8, 0.9, 0.95\}$, the performance improvement remains stable, with no significant degradation. This indicates that TS\_Adam’s update rule is not overly sensitive to moderate changes in $\beta_1$, which enhances its applicability in practice.

In particular, even in conservative settings (such as $\alpha = 5\text{e-}5$), TS\_Adam still yields positive MSE gains, suggesting that its effectiveness does not depend on aggressive tuning. This property is particularly useful in practical forecasting scenarios where hyperparameter search budgets are constrained and stable default settings are desirable.

\paragraph{Impact of batch size on final performance}

\begin{table}[h]
  \centering
  \caption{Evaluating the Impact of Batch Size on Optimizer Performance.}
  \label{tab:batch}
  \resizebox{0.7\textwidth}{!}{
    \begin{tabular}{c|c|cc|cc}
    \toprule
    \multirow{3}[6]{*}{Optimizer} & \multirow{3}[6]{*}{Batch} & \multicolumn{4}{c}{Dataset} \\
\cmidrule{3-6}          &       & \multicolumn{2}{c|}{ETTh1} & \multicolumn{2}{c}{ETTh2} \\
\cmidrule{3-6}          &       & MSE   & MAE   & MSE   & MAE \\
    \midrule
    \multirow{3}[2]{*}{Adam} & 16    & 0.536±0.028 & 0.518±0.015 & 0.616±0.038 & 0.541±0.019 \\
          & 32    & 0.569±0.024 & 0.528±0.011 & 0.581±0.028 & 0.526±0.016 \\
          & 64    & 0.545±0.011 & 0.516±0.006 & 0.586±0.019 & 0.529±0.01 \\
    \midrule
    \multirow{3}[2]{*}{TS\_Adam} & 16    & 0.498±0.014 & 0.497±0.009 & 0.537±0.015 & 0.502±0.009 \\
          & 32    & 0.441±0.005 & 0.460±0.003 & 0.533±0.01 & 0.501±0.005 \\
          & 64    & 0.437±0.003 & 0.455±0.002 & 0.543±0.008 & 0.507±0.005 \\
    \bottomrule
    \end{tabular}%
    }
\end{table}%

We further evaluate the sensitivity of optimizer performance to batch size using the MICN model. Table~\ref{tab:batch} summarizes the MSE and MAE results of Adam and TS\_Adam on ETTh1 and ETTh2 with batch sizes 16, 32 and 64.

Across all configurations, TS\_Adam consistently outperforms Adam, confirming its robustness to mini-batch size variation. Although fluctuations are observed with increasing batch size, the trends are dataset-dependent and do not follow a consistent direction. For example, on ETTh1, performance improves slightly as batch size increases, whereas on ETTh2, optimal results occur at smaller batch sizes.

Quantitatively, TS\_Adam achieves an average reduction in MSE of 0.070 (13.1\%) and a reduction in MAE of 0.044 (8.4\%) over Adam in ETTh1. On ETTh2, the corresponding improvements are 0.059 (10.0\%) in MSE and 0.026 (5.0\%) in MAE. These gains remain stable in all settings, indicating that TS\_Adam maintains its advantage under different levels of gradient estimation noise.

From an application perspective, this robustness is beneficial for time-series forecasting tasks, where batch size is commonly treated as a tunable hyperparameter. TS\_Adam performs consistently well in different batch sizes without requiring extensive tuning, simplifying model selection, and enhancing usability in diverse experimental settings.

\subsubsection{Model Robustness to Data and Training Perturbations}
\label{sec:Model Robustness to Data and Training Perturbations}

The preceding ablations validated the sensitivity of our method to its intrinsic hyperparameters. To further assess the practical utility and generalizability of TS\_Adam in real-world scenarios, we now investigate its robustness to various external and internal perturbations commonly encountered during model training. Real-world time-series data are often contaminated with noise and anomalies, while the choice of training configurations can significantly impact final performance. A robust optimizer should maintain consistent superiority under such non-ideal conditions. Accordingly, we design experiments from three critical perspectives: robustness to data Noise, adaptability to different learning rate scheduling strategies, and stability under varying regularization strengths.

\paragraph{Robustness to Data Noise and Outliers}

\begin{table}[h]
  \centering
  \caption{Performance comparison under data noise and outliers. Results are reported as mean ± standard deviation.}
  \label{tab:noise}
  \resizebox{0.7\textwidth}{!}{
    \begin{tabular}{c|c|cc|cc}
    \toprule
    \multirow{3}[6]{*}{optimizer} & \multirow{3}[6]{*}{noise} & \multicolumn{4}{c}{Dataset} \\
    \cmidrule{3-6}          &       & \multicolumn{2}{c|}{ETTh1} & \multicolumn{2}{c}{ETTh2} \\
    \cmidrule{3-6}          &       & MSE   & MAE   & MSE   & MAE \\
    \midrule
    \multirow{2}[2]{*}{Adam} & Gaussian & 0.54±0.007 & 0.513±0.003 & 0.583±0.005 & 0.527±0.001 \\
          & Outlier & 0.514±0.002 & 0.499±6.7e-4 & 0.569±0.015 & 0.519±0.007 \\
    \midrule
    \multirow{2}[2]{*}{TS\_Adam} & Gaussian & 0.438±0.002 & 0.455±0.001 & 0.543±0.003 & 0.507±0.002 \\
          & Outlier & 0.436±8.9e-4 & 0.455±6.2e-4 & 0.554±0.002 & 0.514±0.001 \\
    \bottomrule
    \end{tabular}%
    }
\end{table}%

As shown in Table~\ref{tab:noise}, TS\_Adam demonstrates significantly better robustness than Adam when the training data are perturbed by noise or outliers. Under both Gaussian noise and extreme outlier settings, TS\_Adam consistently outperforms Adam across the ETTh1 and ETTh2 datasets.

The performance advantage of TS\_Adam is clearly reflected in the lower mean squared error (MSE). For example, on ETTh1 with Gaussian noise, TS\_Adam achieves an MSE of 0.438, compared to 0.540 for Adam—a reduction of nearly 19\%. Moreover, the standard deviations of TS\_Adam’s results are substantially smaller, indicating a more stable and repeatable optimization process that is less affected by data corruption. These observations confirm that the design of TS\_Adam, which aims to improve adaptability to distribution shifts, also enhances its resilience to common data imperfections such as noise and outliers.

\paragraph{Robustness to Learning Rate Scheduling Strategies}

\begin{table}[h]
  \centering
  \caption{Performance under different learning rate scheduling strategies. type1 denotes exponential decay, type2 denotes delayed exponential decay. Results are reported as mean ± standard deviation.}
  \label{tab:lr_strategies}
  \resizebox{0.7\textwidth}{!}{
    \begin{tabular}{c|c|cc|cc}
    \toprule
    \multirow{3}[6]{*}{Optimizer} & \multirow{3}[6]{*}{lr\_strategy} & \multicolumn{4}{c}{Dataset} \\
    \cmidrule{3-6}          &       & \multicolumn{2}{c|}{ETTh1} & \multicolumn{2}{c}{ETTh2} \\
    \cmidrule{3-6}          &       & MSE   & MAE   & MSE   & MAE \\
    \midrule
    \multirow{3}[2]{*}{Adam} & cosine & 0.554±0.005 & 0.522±0.001 & 0.587±0.007 & 0.53±0.003 \\
          & type1 & 0.545±0.008 & 0.516±0.004 & 0.586±0.005 & 0.529±0.001 \\
          & tpye2 & 0.556±0.014 & 0.522±0.006 & 0.587±0.008 & 0.53±0.003 \\
    \midrule
    \multirow{3}[2]{*}{TS\_Adam} & cosine & 0.443±0.004 & 0.46±0.002 & 0.53±0.004 & 0.5±0.002 \\
          & type1 & 0.437±0.001 & 0.455±3.5e-4 & 0.543±0.003 & 0.507±0.002 \\
          & tpye2 & 0.475±0.015 & 0.481±0.008 & 0.531±0.009 & 0.5±0.006 \\
    \bottomrule
    \end{tabular}%

    }
\end{table}%

TS\_Adam demonstrates consistent superiority over Adam across diverse learning rate schedules, as evidenced by the results in Table \ref{tab:lr_strategies}. Its performance advantage remains stable irrespective of the specific scheduling strategy employed.
A key observation is the significant performance gap on both datasets. For instance, under the type1strategy on ETTh1, TS\_Adam achieves an MSE of 0.437, substantially lower than Adam's 0.545. This consistent lead underscores that the efficacy of TS\_Adam is intrinsic to its optimization dynamics, rather than being dependent on a particular learning rate adjustment pattern. The method's robustness ensures reliable performance in practical settings where the optimal scheduling strategy may not be known a priori.

\paragraph{Robustness to Regularization Strength}

\begin{table}[h]
  \centering
  \caption{Performance under different weight‑decay (wd) coefficients. Results are reported as mean ± standard deviation.}
  \label{tab:weight_decay}
  \resizebox{0.7\textwidth}{!}{
    \begin{tabular}{c|c|cc|cc}
    \toprule
    \multirow{3}[6]{*}{Optimizer} & \multirow{3}[6]{*}{wd} & \multicolumn{4}{c}{Dataset} \\
    \cmidrule{3-6}          &       & \multicolumn{2}{c|}{ETTh1} & \multicolumn{2}{c}{ETTh2} \\
    \cmidrule{3-6}          &       & MSE   & MAE   & MSE   & MAE \\
    \midrule
    \multirow{3}[1]{*}{Adam} & 0     & 0.569±0.023 & 0.528±0.01 & 0.581±0.027 &  0.526±0.014 \\
          & 1.00E-04 & 0.569±0.068 & 0.528±0.01 & 0.581±0.027 &  0.526±0.014 \\
          & 1.00E-02 & 0.568±0.023 & 0.528±0.01 & 0.581±0.027 & 0.526±0.014 \\
    \midrule
    \multirow{3}[1]{*}{TS\_Adam} & 0     & 0.441±0.005 & 0.460±0.003 & 0.533±0.01 & 0.501±0.005 \\
          & 1.00E-04 & 0.440±0.001 & 0.460±0.003 & 0.533±0.034 & 0.501±0.005 \\
          & 1.00E-02 & 0.440±0.001 & 0.460±0.003 & 0.534±0.027 & 0.501±0.005 \\
    \bottomrule
    \end{tabular}%
    }
\end{table}%

Table~\ref{tab:weight_decay} shows that weight‑decay has minimal impact on performance for both optimizers, suggesting this regularization technique plays a limited role in the evaluated time series forecasting tasks. Across all tested coefficients (1e-4, 1e-2), MSE values remain almost unchanged for both TS\_Adam and Adam on both datasets.

This insensitivity to weight‑decay may reflect characteristics of the time series data and model architecture, where other forms of implicit regularization or data structure itself already provide sufficient inductive bias. Consequently, the consistent performance advantage of TS\_Adam over Adam is approximately 22\% lower than MSE on ETTh1 and 8\% on ETTh2 apparently stems from optimization dynamics rather than regularization effects.

These findings imply that in practical time series forecasting applications where optimal regularization strength is unknown, TS\_Adam offers more reliable performance without requiring careful weight‑decay tuning.

\subsubsection{Generalization to Other Optimizers}

To assess the generalizability of the TS\_Adam strategy, we apply its core modification, namely the removal of the second-order bias correction term, to other widely used adaptive optimizers, resulting in AdamW\dag, Yogi\dag, and Lookahead\dag. These variants preserve the original update structure of each optimizer while adopting the TS\_Adam adjustment to the effective learning rate.

\begin{table}[h]
  \centering
  \caption{Extending TS\_Adam Strategy to Other Adaptive Optimizers. \dag~denotes the use of the TS\_Adam strategy, which removes second-order bias correction.}
  \label{tab:generalization}
  \resizebox{0.7\textwidth}{!}{
    \begin{tabular}{c|cc|cc}
    \toprule
    \multirow{3}[6]{*}{Optimizer} & \multicolumn{4}{c}{Dataset} \\
\cmidrule{2-5}          & \multicolumn{2}{c|}{ETTh1} & \multicolumn{2}{c}{ETTh2} \\
\cmidrule{2-5}          & MSE   & MAE   & MSE   & MAE \\
    \midrule
    AdamW & 0.568±0.023 & 0.528±0.01 & 0.581±0.027 & 0.526±0.014 \\
    AdamW\dag & 0.441±0.006 & 0.460±0.004 & 0.533±0.012 & 0.501±0.006 \\
    \midrule
    Yogi  & 0.457±0.002 &  0.476±0.001 & 0.572±0.002 & 0.522±0.001 \\
    Yogi\dag & 0.445±0.001 & 0.462±0.001 & 0.561±0.007 & 0.515±0.004 \\
    \midrule
    lookahead & 0.534±0.004 & 0.507±0.002 & 0.575±0.014 & 0.523±0.006 \\
    lookahead\dag & 0.470±0.007 & 0.480±0.004 & 0.542±0.015 & 0.506±0.008 \\
    \bottomrule
    \end{tabular}%
    }
\end{table}%


Table~\ref{tab:generalization} presents the performance comparison on the ETTh1 and ETTh2 datasets using the MICN model. We observe consistent error reductions across all three optimizer variants. Specifically, AdamW\dag achieves reductions of 0.087 in MSE (a relative decrease of 15.3\%) and 0.027 in MAE (a relative decrease of 5.1\%). Yogi\dag shows more modest reductions of 0.012 in MSE (2.2\%) and 0.012 in MAE (2.4\%), while Lookahead\dag reduces errors by 0.049 in MSE (9.0\%) and 0.012 in MAE (2.3\%).

These results highlight two important findings. First, the performance gains are not specific to Adam, suggesting that the TS\_Adam strategy captures a general optimization principle that is effective in non-stationary settings. Second, this adaptability enhances the practical relevance of the method, as it can be readily incorporated into a wide range of bias-corrected adaptive optimizers without structural changes or additional tuning effort. Such flexibility is particularly valuable in real-world forecasting applications, where different optimizers may be preferred based on model characteristics or task-specific considerations.

\subsubsection{Empirical Verification of Cumulative Regret Upper Bound}

\begin{figure*}[ht]
\vskip 0.2in
\begin{center}
\includegraphics[width=0.8\textwidth]{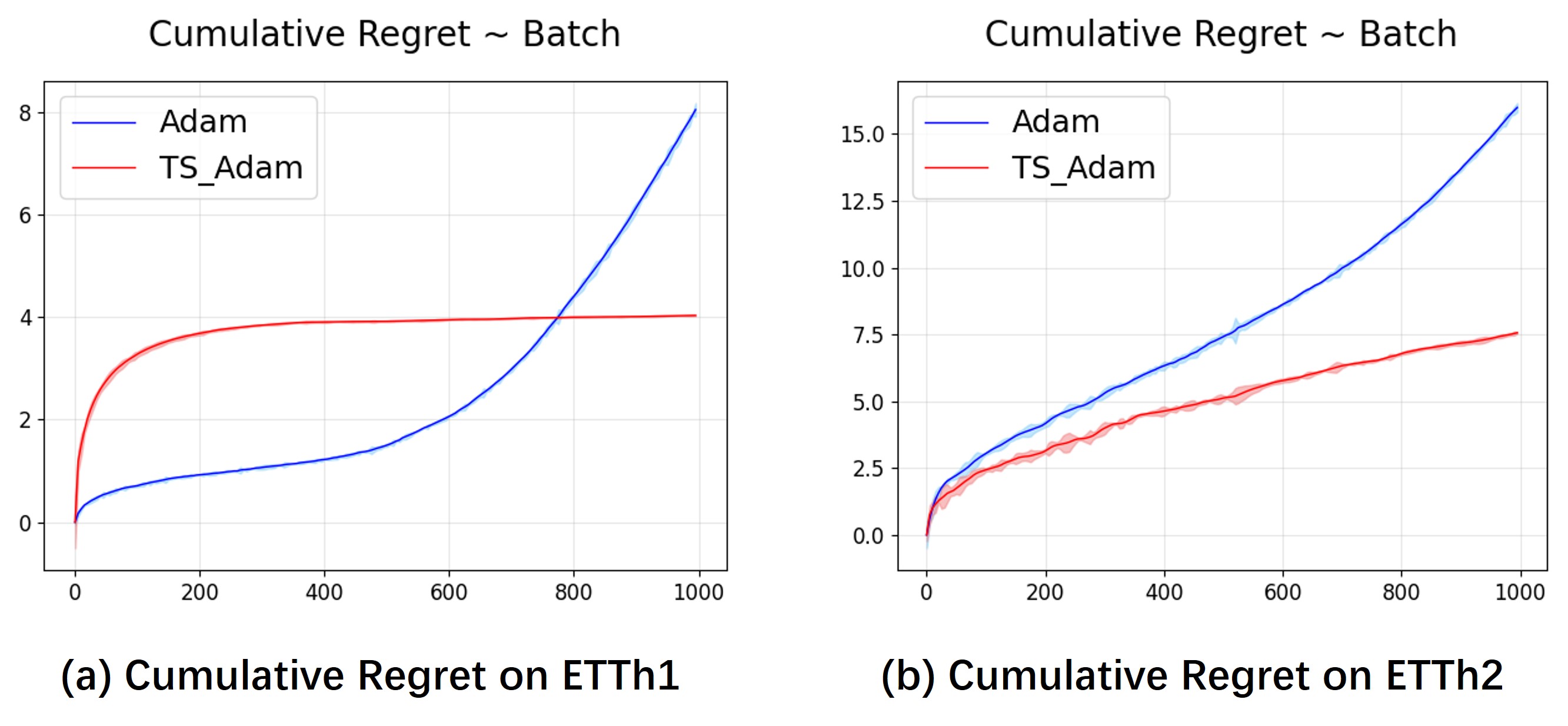}
\caption{Cumulative regret curves on ETTh1 and ETTh2 using the MICN model with prediction length 192. Each curve shows the mean regret over three runs, and the shaded area denotes $\pm 5$ standard deviations to enhance visibility. (a) ETTh1; (b) ETTh2.} 
\label{fig:Regret}
\end{center}
\vskip -0.2in
\end{figure*}

To empirically validate our theoretical insight that suppressing regret from distributional drift is more critical than mitigating noise-induced regret in nonstationary time series forecasting, we compare the cumulative regret of Adam and TS\_Adam during training.

We follow the standard online learning formulation~\cite{zhang2018adaptive} to compute cumulative regret $R(T)$ in training batches:
\begin{equation}
R(T) = \sum_{t=1}^T \left[ f_t(\theta_t) - f_t(\theta_t^*) \right],
\end{equation}
where $f_t(\cdot)$ denotes the test loss at step $t$, $\theta_t$ is the model parameter at that step, and $\theta_t^*$ corresponds to the parameter achieving the lowest test loss across training. For each optimizer, cumulative regret is computed per run and averaged over three runs. The resulting curves are shown in Fig~\ref{fig:Regret}.

TS\_Adam consistently accumulates less cumulative regret than Adam over the course of training. Although it shows slightly higher regret at the very beginning, this difference disappears within a few epochs as its improved adaptability to distributional shifts takes effect. TS\_Adam more effectively suppresses drift-induced regret, resulting in lower overall regret and better final performance. In contrast, although Adam benefits from early noise correction, its cumulative regret increases in a highly unstable manner, especially on ETTh2, which poses a concern for real-world forecasting systems that require stable and reliable optimization behavior. These findings support our theoretical claim that, under distributional drift, controlling regret from shifting data distributions is more critical than reducing early-stage noise.

\section{Conclusion}
\label{sec:Conclusion}

In this paper, we have proposed TS\_Adam, a simple but effective optimizer for time-series forecasting. By removing the second-order bias correction from the learning rate computation, TS\_Adam improves responsiveness to distributional drift while preserving the robustness of stochastic gradient updates.

Comprehensive experiments on long- and short-term forecasting benchmarks demonstrate that TS\_Adam consistently outperforms baseline optimizers such as Adam, AdamW, Yogi and Lookahead in diverse models and datasets. These improvements are particularly valuable in practical settings where predictive accuracy, stability, and adaptability are essential for reliable temporal modeling.

Ablation studies further confirm the stability of TS\_Adam with respect to hyperparameter changes and its applicability in different optimization scenarios. In addition, empirical analysis of cumulative regret supports the theoretical insights, reinforcing the importance of mitigating drift-induced regret in non-stationary environments.

While TS\_Adam achieves notable gains on datasets with pronounced periodicity, its effectiveness on trend-dominated series remains relatively limited. Future work will explore strategies to further improve adaptability to persistent trends.

In summary, TS\_Adam offers a lightweight, generalizable optimization strategy that achieves strong forecasting performance under real-world, dynamic conditions. Its simplicity and consistent results make it a practical choice for time series forecasting tasks requiring both accuracy and robustness.

\section{Acknowledgments}
This work was supported in part by China Guangxi Science and Technology Plan Project under grant AD23026096, Chile CONICYT FONDECYT Regular under grant 1181809, and Chile CONICYT FONDEF under grant ID16I10466.

\appendix
\section{Proof of Theorem 1}

\begin{proof}[{ Proof } ]
We derive the distributional properties of $Y_t$ by analyzing its expectation, variance, and Gaussian structure. Specifically:

{Expectation:}
Using the zero-mean assumptions $\mathbb{E}[\varepsilon_i] = \mathbb{E}[R_t] = 0$ and linearity of expectation:
\[
\mathbb{E}[Y_t] = c + \sum_{i=2}^{t}\mathbb{E}[\varepsilon_i] + S_t + \mathbb{E}[R_t] = c + S_t.
\]

{Variance:}
By independence between $\{\varepsilon_i\}$ and $R_t$, and the deterministic nature of $S_t$:
\[
\operatorname{Var}(Y_t) = \operatorname{Var}\left(\sum_{i=2}^{t}\varepsilon_i\right) + \operatorname{Var}(R_t) = (t-1)\sigma_\varepsilon^{2} + \sigma_R^{2}.
\]

{Distribution:}
Since $\varepsilon_i \sim \mathcal{N}(0,\sigma_\varepsilon^2)$ and $R_t \sim \mathcal{N}(0,\sigma_R^2)$ are independent Gaussian variables:
\[
Y_t \sim \mathcal{N}\left(c + S_t,\ (t-1)\sigma_\varepsilon^{2} + \sigma_R^{2}\right),
\]
where both mean and variance are explicit functions of $t$.
\end{proof}

\section{Proof of Theorem 2}

\noindent\textbf{Notation and assumptions:} 
All quantities are in $\mathbb{R}^d$: $\theta_t, \theta_t^*, \nabla f_t(\theta_t) \in \mathbb{R}^d$. 
The norm $\|\cdot\|$ denotes the Euclidean norm. 
The assumptions are: (1) convexity of $f_t$, 
(2) gradient bound $\|\nabla f_t(\theta)\| \leq G$, 
(3) drift condition $\|\theta_{t+1}^* - \theta_t^*\| \leq \delta$.

\begin{proof}
Consider the online gradient descent algorithm $\theta_{t+1} = \theta_t - \eta_t \nabla f_t(\theta_t)$ with dynamic regret $R(T) = \sum_{t=1}^T [f_t(\theta_t) - f_t(\theta_t^*)]$.

By convexity of $f_t$:
\begin{align*}
f_t(\theta_t) - f_t(\theta_t^*) \leq \langle \nabla f_t(\theta_t), \theta_t - \theta_t^* \rangle.
\end{align*}

Decompose $\theta_t - \theta_t^* = (\theta_t - \theta_{t-1}^*) + (\theta_{t-1}^* - \theta_t^*)$:
\begin{align*}
f_t(\theta_t) - f_t(\theta_t^*) \leq \langle \nabla f_t(\theta_t), \theta_t - \theta_{t-1}^* \rangle + \langle \nabla f_t(\theta_t), \theta_{t-1}^* - \theta_t^* \rangle.
\end{align*}

Using identity $2\langle a, b \rangle = \|a\|^2 + \|b\|^2 - \|a - b\|^2$:
\begin{align*}
2\langle \nabla f_t(\theta_t), \theta_t - \theta_{t-1}^* \rangle &= \|\nabla f_t(\theta_t)\|^2 + \|\theta_t - \theta_{t-1}^*\|^2 \\
&\quad- \|\theta_t - \theta_{t-1}^* - \nabla f_t(\theta_t)\|^2.
\end{align*}

From update rule $\theta_{t+1} = \theta_t - \eta_t \nabla f_t(\theta_t)$:
\begin{align*}
\theta_t - \theta_{t-1}^* - \eta_t \nabla f_t(\theta_t) = \theta_{t+1} - \theta_{t-1}^*.
\end{align*}

Therefore:
\begin{align*}
\langle \nabla f_t(\theta_t), \theta_t - \theta_{t-1}^* \rangle &= \frac{1}{2\eta_t} \left[ \|\theta_t - \theta_{t-1}^*\|^2 - \|\theta_{t+1} - \theta_{t-1}^*\|^2 \right]\\
&\quad+ \frac{\eta_t}{2} \|\nabla f_t(\theta_t)\|^2.
\end{align*}

By Cauchy-Schwarz inequality and assumptions $\|\nabla f_t(\theta_t)\| \leq G$, $\|\theta_{t-1}^* - \theta_t^*\| \leq \delta$:
\begin{align*}
\langle \nabla f_t(\theta_t), \theta_{t-1}^* - \theta_t^* \rangle &\leq |\nabla f_t(\theta_t)| \cdot |\theta_{t-1}^* - \theta_t^*| \leq G \delta.
\end{align*}

Combining results:
\begin{align*}
f_t(\theta_t) - f_t(\theta_t^*) &\leq \frac{1}{2\eta_t} \left[ \|\theta_t - \theta_{t-1}^*\|^2 - \|\theta_{t+1} - \theta_{t-1}^*\|^2 \right] + \frac{\eta_t G^2}{2} + G\delta.
\end{align*}

Summing over $t = 1, \ldots, T$:
\begin{align*}
R(T) &\leq \underbrace{\sum_{t=1}^T \frac{1}{2\eta_t} \left[ \|\theta_t - \theta_{t-1}^*\|^2 - \|\theta_{t+1} - \theta_{t-1}^*\|^2 \right]}_{(A)} + \frac{G^2}{2} \sum_{t=1}^T \eta_t + GT\delta.
\end{align*}

Expanding the telescoping sum:
\begin{align*}
(A)
&= \frac{1}{2\eta_1} \|\theta_1 - \theta_1^*\|^2 + \sum_{t=2}^T \frac{1}{2\eta_t} \|\theta_t - \theta_{t-1}^*\|^2 \\
&\quad- \sum_{t=1}^{T-1} \frac{1}{2\eta_t} \|\theta_{t+1} - \theta_{t-1}^*\|^2 - \frac{1}{2\eta_T} \|\theta_{T+1} - \theta_{T-1}^*\|^2.
\end{align*}


\textbf{Bounding $\|\theta_{t+1} - \theta_{t-1}^*\|^2$:}
We start from the decomposition:
\[
\theta_{t+1} - \theta_{t-1}^* = (\theta_{t+1} - \theta_t^*) + (\theta_t^* - \theta_{t-1}^*).
\]
Expanding the squared norm:
\begin{align*}
\|\theta_{t+1} - \theta_{t-1}^*\|^2 
&= \|\theta_{t+1} - \theta_t^*\|^2 + \|\theta_t^* - \theta_{t-1}^*\|^2 + 2\langle \theta_{t+1} - \theta_t^*, \theta_t^* - \theta_{t-1}^* \rangle \\
&= \|\theta_{t+1} - \theta_t^*\|^2 + \delta^2 + 2\langle \theta_{t+1} - \theta_t^*, \theta_t^* - \theta_{t-1}^* \rangle \quad \text{(since $\|\theta_t^* - \theta_{t-1}^*\| \leq \delta$)}.
\end{align*}

Applying Cauchy-Schwarz to the inner product:
\begin{align*}
\langle \theta_{t+1} - \theta_t^*, \theta_t^* - \theta_{t-1}^* \rangle 
&\geq -\|\theta_{t+1} - \theta_t^*\| \cdot \|\theta_t^* - \theta_{t-1}^*\| \\
&\geq -\|\theta_{t+1} - \theta_t^*\| \cdot \delta.
\end{align*}

From the update rule $\theta_{t+1} = \theta_t - \eta_t \nabla f_t(\theta_t)$ and gradient bound:
\[
\|\theta_{t+1} - \theta_t^*\| = \eta_t \|\nabla f_t(\theta_t)\| \leq \eta_t G.
\]

Thus, we obtain the lower bound:
\begin{align*}
\|\theta_{t+1} - \theta_{t-1}^*\|^2 
&\geq \|\theta_{t+1} - \theta_t^*\|^2 + \delta^2 - 2\eta_t G\delta \\
&\geq \|\theta_{t+1} - \theta_t^*\|^2 - 2G\delta + \delta^2 \quad \text{(since $\eta_t \leq 1$ for stable learning rates)}.
\end{align*}
This inequality is used in the telescoping sum to bound the negative terms.

After telescoping manipulation and applying bounds:

\begin{align*}
(A) &\leq \frac{1}{2\eta_1} \|\theta_1 - \theta_1^*\|^2 + \delta \sum_{t=1}^T \frac{1}{\eta_t} + G\delta T - \frac{\delta^2 T}{2}
\end{align*}


\textbf{Note on the $\delta^2$ term:} 
In the final bound, we retain $-\frac{\delta^2 T}{2}$ for completeness. 
When $\delta$ is small (typical in slowly varying environments), this term is dominated by $G\delta T$, 
but we include it to show the full derivation. The condition $\delta \ll G$ ensures $G\delta - \frac{\delta^2}{2} > 0$.

Substituting back:
\begin{align*}
R(T) &\leq \frac{1}{2\eta_1} \|\theta_1 - \theta_1^*\|^2 + \frac{G^2}{2} \sum_{t=1}^T \eta_t + \delta \sum_{t=1}^T \frac{1}{\eta_t} + G\delta T - \frac{\delta^2 T}{2}
\end{align*}
\end{proof}

\bibliography{refer}
\bibliographystyle{elsarticle-harv}




\end{document}